\documentclass[12pt]{l4dc2023}

\title[Concentration Phenomenon: an  Operator Theoretic Approach]{Concentration Phenomenon for Random Dynamical Systems: 
An Operator Theoretic Approach}
\usepackage{times}

\author{\Name {Muhammad Abdullah Naeem} 
\Email{muhammad.abdullah.naeem@duke.edu} \\
\addr Department of Electrical and Computer Engineering\\
Duke University\\
Durham, NC 27708, USA
\AND
\Name {Miroslav Pajic} \Email{miroslav.pajic@duke.edu} \\
\addr Department of Electrical and Computer Engineering\\
Duke University\\
Durham, NC 27708, USA}

\usepackage{todonotes}

\begin{document}

\maketitle

\begin{abstract}%
Via operator theoretic methods, we formalize the concentration phenomenon for a given observable `$r$' of a discrete time Markov chain with `$\mu_{\pi}$' as invariant ergodic measure, possibly having support on an unbounded state space. The main contribution of this paper is circumventing tedious probabilistic methods with a study of a composition of the Markov transition operator $P$ followed by a multiplication operator defined by $e^{r}$. It turns out that even if the observable/ reward function  is unbounded, but for some for some $q>2$, $\|e^{r}\|_{q \rightarrow 2} \propto \exp\big(\mu_{\pi}(r) +\frac{2q}{q-2}\big) $ and $P$ is hyperbounded with norm control $\|P\|_{2 \rightarrow q }< e^{\frac{1}{2}[\frac{1}{2}-\frac{1}{q}]}$, sharp non-asymptotic concentration bounds follow. \emph{Transport-entropy} inequality ensures the aforementioned upper bound on multiplication operator for all $q>2$.  Also, the role of \emph{reversibility} in concentration phenomenon is demystified. These results are particularly useful for the reinforcement learning and controls communities by allowing for concentration inequalities w.r.t standard unbounded obersvables/reward functions where exact knowledge of the system is not available, let alone the reversibility of stationary measure.               %
\end{abstract}


\section{Introduction}
\paragraph{Motivation.}
Successful use of control theory in high stake applications stems from its rigorous theoretical foundations; properties of dynamical system under consideration can be extracted from the spectral analysis (eigenvalues/eigenvector) of an associated matrix. The last decade has seen a tremendous surge in research activity on non-asymptotic analysis/concentration phenomenon of system identification 
and reinforcement learning for random dynamical systems (e.g.,~\cite{tu2018least,hao2020provably, oymak2019stochastic,fazel2018global,simchowitz2018learning,sarkar2019finite,zahavy2019average,ziemann2022learning}); tedious probabilistic techniques are employed making analysis intractable. If not, assumptions like boundedness of obervables, discounting the cost, or discretizing the state space are 
employed. 
To address these limitations, there is a need to 
bring concentration phenomenon for dynamical systems on a same theoretical footing with~controls.

\paragraph{Related Literature and Outline of Paper.} Whether it is a question of learning value function corresponding to a stabilizing policy for a Markov decision process~(MDP) 
(\cite{tu2018least}) or system identification (e.g.,~\cite{sattar2020non,foster2020learning,simchowitz2018learning}), the existing analysis techniques heavily rely 
on mixing arguments to conclude `time-distant temporally dependent samples  behave as independent identically distributed (iid)'. However, the notion of time-distant samples being uncorrelated is equivalent to \emph{$L^2$ spectral gap or Poincare inequality} (see e.g., Chapter~7 of~\cite{rosenblatt2012markov}). Secondly, probabilistic analyses are tedious, opaque, and leading to weak results (strong results are often limited to unrealistic assumptions). 
 
Recognizing these limitation, \emph{we focus on non-asymptotic analysis of policy evaluation for average reward-based control on continuous state spaces, where the expected value of the reward w.r.t the stationary distribution of an MDP is approximated by its empirical averages, with high probability.} 
Our approach is inspired from a conjecture put forth by \cite{simon1972hypercontractive} 
regarding spectral gaps for hyperbounded Markov semi-groups, which was recently solved by \cite{miclo2015hyperboundedness} for reversible case and by \cite{gluck2020spectral} in full generality. A remarkable advantage of this approach is uplifting a nonlinear random dynamical system into an infinite dimensional space where it behaves as a linear system. 
%
We then, 
as described in Section \ref{sec:secsgap},
employ direct sum decomposition of the underlying Banach space, as in \cite{lin1974uniform}, to prove the uniform ergodic theorem for the Markov transition operator. 
Consequently, we 
show how hyperboundedness leads to the desired spectral properties (given in \cite{luecke1977norm}) for uncorellation of Poincare Inequality).

To the best of our knowledge, \cite{kloeckner2017effective,kloeckner2019effective} interpreted the spectral gap in Markov transition operator
into non-asymptotic concentration bounds but his analysis only works on bounded observables. In Section \ref{sec:Hyp}, we introduce Feynman-Kac semi-group related to unbounded observable, which 
based on 
spectral properties of the multiplication operator 
allows us to derive a sufficient condition for sharp concentration: \emph{transport-entropy inequality}; we then 
verify our results on linear Gaussian systems. \cite{wang2020transport} offer an operator theoretic treatment for concentration but pertinent functional inequalities are valid only for reversibility assumption.
Finally, concluding remarks are made in Section \ref{sec:conclusion} along with discussion on future problems for consideration.
\paragraph{Further implications of the analysis: estimating steady state correlations via sample covariance matrices.} 
Let us assume that we are observing i.i.d samples $(x_{i})_{i=1}^{N}$ from a centered distribution $P_{\infty}$ on $\mathbb{R}^{n}$, where $n$ is high dimensional. For example, in finance, each $x_{i}$ could be return on $n$ stocks and $N$ the number of trading days~\cite{bun2017cleaning}. One wishes to estimate $P_{\infty}$ via $\Sigma_{N}:=\frac{1}{N}\sum_{i=0}^{N-1}x_{i}x_{i}^{T}$; in order to understand most significant factors and spatial correlations between the $n$ variables. It is well known in statistics literature that for $N>>n$, empirical covariance matrix $\Sigma_{N}$ is a good approximation of $P_{\infty}$. Now, consider covariates $(x_{i})_{i=1}^{N}$ as realizations of a random dynamical system in steady state; good estimate of $P_{\infty}$ is imperative for designing reduced order models, after performing principal component analysis. However, samples are now temporally correlated, along with spatial correlations. Thus, one would expect sample complexity to be worse than in the i.i.d case. Intuitively, one needs to quantify decay of temporal correlation/spectral gaps (as we do in Section \ref{sec:secsgap}), and then average only over 'distant' samples to get an accurate estimation of the stationary covariance matrix. Analysis developed in this paper can be extended, as part of future work, to estimation of correlations where the observable would be instead $f(x)=x x^{T}$.
\paragraph{Notations.}
The space of probability measures on a metric space $(\mathcal{X},d)$ is denoted by  $\mathcal{P(\mathcal{X})}$. 
$\mathbb{E}$ is used to denote expectation. For a function $r$, we use $<r>_{\mu}$ to denote the expectation of $r$ w.r.t $\mu$. Given $\mu, \nu \in \mathcal{P(X)}$, Wasserstein metric of order $p \in [1, \infty)$ is defined as:
\begin{equation}
\label{eq:WM}
    \mathcal{W}_{d}^{p} (\nu,\mu)= \bigg(\inf_{(X,Y) \in \Gamma(\nu,\mu)} \mathbb{E}~d^{p}(X,Y)\bigg)^{\frac{1}{p}};
\end{equation}
here, $\Gamma(\nu,\mu) \in \mathcal{P}(\mathcal{X}^{2})$, and $(X,Y) \in \Gamma(\nu,\mu)$ denotes that random variables $(X,Y)$ follow some probability distributions on $\mathcal{P}(\mathcal{X}^{2})$ with marginals $\nu$ and $\mu$. Another way of comparing two probability distributions on $\mathcal{X}$ is via relative entropy which is defined as:
\begin{equation}
    \label{eq:ent}  
    Ent(v||u)= \begin{cases}
    \int \log\bigg(\frac{d\nu}{d\mu}\bigg) d\nu, & \text{if}~ \nu << \mu, \\
        +\infty, & \text{otherwise}. 
    \end{cases}
\end{equation}

\subsection{Problem Statement}
Under the action of some state-dependent  control policy $\pi$, we consider a closed-loop random dynamical system of the form:
\begin{equation}
\label{eq:crds}
    x_{k+1}=F\big(x_k, \pi(x_k), \epsilon_k\big), \hspace{10 pt} \epsilon_k  \hspace{6 pt} \text{is}~~iid, \footnote{For the sake of brevity, from now on, we will exclude reference to $\pi$ in the state update equations as state dependent policy imlplies there exists some function $G$ such that $F\big(x_k, \pi(x_k), \epsilon_k\big)=G(x_k, \epsilon_k)$.}
\end{equation}
where $x_{k} \in \mathbb{R}^n$ for all $k \in \mathbb{N}$ and $F: \mathbb{R}^n \times \mathbb{R}^n \times \mathbb{R}^n \longrightarrow \mathbb{R}^n $. In probabilistic language, the closed-loop dynamical system is a \emph{Harris Ergodic Markov chain}.
Let us assume that the Markov chain converges to a unique ergodic invariant measure $\mu_{\pi}$. A question that is of utmost importance in average reward reinforcement learning  when exact dynamics are unknown (\cite{zahavy2019average}) is: if we have access to empirical averages of some unbounded reward function $r(x)$, 
can we estimate the \emph{concentration from simulating a single trajectory} -- i.e., when, how, and why can we provide something similar to following exponential concentrations
\begin{align}
\label{eq:expcon}  \mu^{N} \Bigg[\bigg|\frac{1}{N} \sum_{i=1}^{N} r(x_i) - <r>_{\mu_{\pi}} \bigg| > 
\epsilon \Bigg]  \leq 2 \exp \bigg(-\frac{N \epsilon^2}{K_{sys}(r)}\bigg)?
\end{align}
Here, $r$ can be some unbounded function (e.g., in control theoretic or (RL) framework $r(x):=\|x\|$), $K_{sys}(r)$ is a constant dependent on system properties and `smoothness' of $r$ (related to Lipschitz constant) and $\mathcal{\mu}^{N}$ denotes the probability on $N$ covariates of the Markov Chain $(x_1,\ldots,x_N)$. 
     
\subsection{Contribution and Main Results}
This paper's fundamental contribution is reducing the concentration phenomenon for Harris Ergodic Markov chains to a problem of bounding operator $e^{r}P : L^{2}(\mu_{\pi}) \longrightarrow L^{2}(\mu_{\pi})$, a composition of Markov operator $P$ and a multiplication operator $e^r$. A detailed mathematical and pedagogical overview of \emph{Poincare/ Spectral gap inequality} is provided; we show that the problem is equivalent to temporally dependent samples having a concentration similar to iid samples. Even though multiplication operator $e^{r}$, associated with the standard reward function (observable) used in continuous control and RL problems is \emph{unbounded}, \emph{transport-entropy inequality} ensures that for all $q>2$, $e^{r}: L^{q}(\mu_{\pi}) \longrightarrow L^{2}(\mu_{\pi}) $ is bounded. Combined with \emph{hyperboundedness} of Markov kernel, the concentration phenomenon is achieved.

\section{Spectral Gaps, Ergodic Theorems and Poincare Inequality} 
\label{sec:secsgap}
To mathematically express the phenomenon that \emph{time-distant samples, although temporally dependent, behave as iid }, we will have to go to fundamentals of ergodic theory. Let us use $\mathbb{D}$ to denote the unit disc in the complex plane and $\mathbb{T}$ represents the unit circle in the plane. Consider the Banach space, with complex field, $L^{p}(\mu_{\pi})$ of $p \in [1, \infty]$ integrable functions; i.e., all measurable functions $f$ with $\int |f(x)|^p \mu_{\pi}(dx) < \infty $. Conjugate power of $p$, $c(p)$ is defined as $c(p) \geq 1$ that satisfies $\frac{1}{c(p)} +\frac{1}{p}=1$. Every $g \in L^{c(p)}$ corresponds to a bounded linear functional on $f \in L^{p}$ and its action is captured by $\langle f, g \rangle _{\mu_{\pi}} := \int f(x)\overline{g(x)} d \mu_{\pi}(x)$. 

We consider linearity in first argument of the inner product and the boundedness follows by Cauchy-Schwarz: $|\langle f, g \rangle _{\mu_{\pi}}| \leq \|f\|_{p}\|g\|_{c(p)}$. Markov transition operator $P:L^{p} \rightarrow L^{p}$ is defined as $\|Pf\|_{L^p}= \bigg(\int |Pf(x)|^{p} \mu_{\pi}(dx)\bigg)^{\frac{1}{p}}= \bigg(\int |\int f(y)p(x, dy) |^{p} \mu_{\pi}(dx)\bigg)^{\frac{1}{p}} \leq \|f\|_{L^{p}}$, where the last inequality follows from Jensen inequality. Note that $P$ acts as an identity on constant functions, i.e., $P1=1$, and it is positive, i.e., $Pf \geq 0 $ if $f \geq 0 $. 

A Markov chain is \emph{reversible or satisfies detailed balance condition} if Markov transition operator $P$ when viewed as an operator on Hilbert space $L^{2}(\mu_{\pi})$  is equal to its  adjoint $P^*$: $\langle Pf,g \rangle:= \langle f, P^{*}g \rangle _{\mu_{\pi}} = \langle f, Pg\rangle$. A simple observation reveals that $P, P^*, PP^*$ and $P^*P$ are all Markov operators with $\mu_{\pi}$ as invariant measure. Consider a sequence $(A_{n})_{n \in \mathbb{N}}$ of bounded operators on some Banach space, there are different topologies (uniform, strong, and weak) to study its convergence (see e.g., Chapter 6 of \cite{reed1980functional} for exact definitions). We also suggest the reader unfamiliar with spectral properties of operators and associated notations (particularly Hyperboundedness and Fredholm theory) to consult Appendix~\ref{sec:appendix}.

\subsection{Ergodic Theorems and Consequences}

\begin{definition}
In functional analytic framework, the pair $(P, \mu_{\pi})$ is said to be ergodic if for any $f \in L^{\infty}(\mu_{\pi})$ satisfies  $Pf=f$, then f is a constant. 
\end{definition}

 In probabilistic language, an ergodic Markov chain has a unique invariant stationary distribution.

\begin{definition}
The pair $(P, \mu_{\pi})$ is called aperiodic if for all $\lambda \in \mathbb{T}\setminus\{1\}$, $dim[N(\lambda I-P)]=0$. 
\end{definition}

\begin{theorem}
{Birkhoff Pointwise Ergodic theorem:} Let $(x_n)_{n \in \mathbb{N}}$, be the samples from an  ergodic Markov chain. Then,
$\frac{1}{N} \sum_{n=0}^{N-1} f(x_n) \longrightarrow \mu_{\pi}(f)$, {almost every initial condition}  $x_0=x$  w.r.t $\mu_{\pi},$ and $f$ integrable w.r.t $\mu_{\pi}$.
\end{theorem}

This is reminiscent of the strong law of large numbers for iid sampling from distribution $\mu_{\pi}$. However, a very natural requirement for having a concentration similar to iid setting is sharp decay of correlation -- i.e., for some $C<\infty$ and $\eta \in (0,1)$
\begin{equation}
    \label{eq:pointran} |Cov_{\mu_{\pi}}[f(x_n),f(x_{n+m})]| \leq C \eta^{m} Var_{\mu_{\pi}} (f), \hspace{10pt} \forall f \in L^{2}(\mu_{\pi}).
\end{equation}

Now, the first step towards formalizing the preceding phenomenon is a concept related to \emph{uniform ergodicity}. 

\begin{theorem}
If the pair $(P, \mu_{\pi})$ is ergodic, let $U$ be the orthogonal projection on $\{f \in L^{2}(\mu_{\pi}) : Pf=f \}$. Then, 
\begin{equation}
\label{eq:meanerg} 
\| \frac{1}{N} \sum_{n=0}^{N-1} P^{n}- U\|_{L^{2} (\mu_{\pi})} \longrightarrow 0;
\end{equation}
i.e., the operator, $\frac{1}{N} \sum_{n=0}^{N-1} P^{n}$
converges to a bounded linear projection $U$ in  the uniform operator topology.
\end{theorem}
\begin{proof}
If the pair $(P,\mu_{\pi})$ is ergodic then so is the pair $(P^*,\mu_{\pi})$.
We first show that $Im(I-P)$ is closed using \emph{Fredholm argument} -- because the dimension of \emph{co-kernel} is finite: $dim[L^{2}(\mu_{\pi})\setminus Im(I-P)]= dim[N(I-P^*)]= dim[N(I-P)]=1$. 
Now, consider the following decomposition:
\begin{equation}
    \label{eq:freddirectsum}
    (I-P):N(I-P) \bigoplus N(I-P)^{\perp} \rightarrow Im(I-P) \bigoplus Im(I-P)^{\perp}.  
\end{equation}
Therefore, $I-P: N(I-P)^{\perp} \rightarrow Im(I-P)$ is bijective and by inverse mapping theorem $(I-P)^{-1} \in \mathcal{B}(Im(I-P),N(I-P)^{\perp})$ so there exists a $K<\infty$ such that $\|(I-P)^{-1}\|_{Im(I-P)\rightarrow N(I-P)^{\perp} } \leq K$ implying that the pair $(P,\mu_{\pi})$ is \emph{uniformly ergodic} because: given any $f \in Im(I-P)$ there exists a unique $\hat{g} \in N(I-P)^{\perp}$  such that $f=(I-P)\hat{g}$ and $\|\frac{1}{N} \sum_{n=0}^{N-1} P^{n}f \|= \|\frac{1}{N} \sum_{n=0}^{N-1} P^{n}(I-P)\hat{g} \|=\frac{\|(I-P^N)\hat{g}\|}{N}=\frac{\|(I-P^N)(I-P)^{-1}f\|}{N} \leq 2\frac{\|(I-P)^{-1}f\| }{N} \leq  \frac{2K}{N} \|f\|$. So $\frac{1}{N} \sum_{n=0}^{N-1} P^{n}$ converges to $0$ at a uniform rate on $Im(I-P)$. Notice that,
when $Im(I-P)$ is closed: $Im(I-P)=N(I-P)^{\perp}$, because regardless of $Im(I-P)$ being closed : $Im(I-P) \subset N(I-P^*)^{\perp}$ and $Im(I-P)^{\perp}=N(I-P^*)$. Therefore, $Im(I-P)^{\perp}$ corresponds to $N(I-P)$ which by ergodicity assumption only comprises of constant function and $\frac{1}{N} \sum_{n=0}^{N-1} P^{n} _{|_{N(I-P)}} =I$.
\end{proof}
\begin{remark}
 The above theorem effectively provide a stronger version of the mean ergodic theorem -- i.e., `$L^2 -$ uniform ergodicity'.
\end{remark}
\vspace{-10pt}
\begin{remark}
It is worth noting the following:
\begin{enumerate}
    \item Under ergodicity and reversibility assumption, $I-P$ is an \emph{Fredholm operator } with index 0; effectively, it means $1$ is an isolated eigenvalue of $P$ -- i.e., for some $\epsilon>0$, $\sigma(P)=[-1,1-\epsilon] \cup \{1\}$ and the dimension of space of eigenfunctions associated to $1$, is $1$ .
    \item We define $S:=P_{|Im(I-P)}$ (when $Im(I-P)$ is closed), and since $Im(I-P)$ is invariant under $P$, implying $S:Im(I-P) \rightarrow Im(I-P) $. Because $(I-P)^{-1} \in \mathcal{B}(Im(I-P))$, $(I-S)^{-1}$ exists as well so $1 \not \in \sigma(S)$ and from previous bullet point as $S$ will contain subset of spectrum of $P$, indeed $\sigma(S) \subset [-1,1-\epsilon]$.  
\end{enumerate}
 
\end{remark}
\begin{theorem}
\label{thm:opcontoproj}
If the pair $(P, \mu_{\pi})$ is reversible, ergodic and aperiodic, then
$\|P^{n}-U\|_{L^{2}(\mu_{\pi})} \longrightarrow 0,$ 
in uniform operator topology.
\end{theorem}
\begin{proof}
Consider the operator $I+P$, as a map
\begin{equation}
    \label{eq:fredap} I+P: N(I+P) \bigoplus N(I+P)^{\perp} \rightarrow Im(I+P) \bigoplus Im(I+P)^{\perp}.
\end{equation}
$N(I+P)=\{f \in L^{2}(\mu_{\pi}):Pf=-f\}$, which corresponds to \emph{periodic} behavior but by hypothesis $dim[N(I+P)]=0$ and $dim[L^{2}(\mu_{\pi}) \setminus Im(I+P)]=dim[N(I+P)]=0$ (thanks to reversibility) so $I+P$ is a Fredholm operator implying that $-1 \not \in \sigma_{ess}(P)$ (see e.g., \cite{wu2004essential}). Moreover, $-1 \not \in \sigma(P)$ either. Since the spectrum of a bounded operator is always closed, there exists $\delta>0$ such that $\sigma(S)\subset [-1+\delta,1-\epsilon]$. Consequently $\rho(S)$, the spectral radius of $S$,  satisfies $\rho(S) \leq max(|-1+\delta|,|1-\epsilon|)<1$ and $P^n=S^n \bigoplus I_{N(I-P)}$, but $\rho(S)<1$ implies $S^n \rightarrow 0$, because by Gelfand's formula for any $\rho \in (\rho(S),1)$ there exists $M(\rho)< \infty$ such that $\|S^n\| \leq M(\rho) \rho^{n}$. 
\end{proof}
Spectral analysis of non-reversible $P$ is more involved as it can lies inside a unit disc and this is where \emph{hyperboundedness} comes into play. 

\begin{theorem}
\label{thm:Tnotessp}
Hyperboundedness(see definition \ref{def:hypbd} in Appendix) of $L^2-L^p$, where $p =2+\epsilon$ and $\epsilon>0$ implies that (a)~for all $\lambda \in \mathbb{T}$, $dim[N(\lambda I-P)_{2}]<\infty$ and $(b)$ $dim[N(\lambda I-P^*)_{2}]<\infty.$ (b)~{independently implies} $Im(\lambda I-P)_{2}$ is closed. 
Consequently:
\begin{enumerate}
    \item $\sigma(P) \cap \mathbb{T}$ may only comprise of a finite number of distinct eigenvalues and each distinct eigenvalue has a finite-dimensional eigenspan.  
    \item $\sigma_{ess}(P)$ is contained inside some closed disc of radius $\alpha<1$ in the complex plane.
    \item $\sigma(P)$ is `gapped' -- i.e., it comprises of two disjoint sets: $\sigma(P)= \{\sigma(P)\cap\mathbb{T}\} \cup \{\sigma(P) \cap \alpha \mathbb{D} \}$.
\end{enumerate}
\end{theorem}
\begin{proof}
 Our proof is based on a result of \cite{gluck2020spectral}: infinite dimensional $L^{p}$ spaces are not isomorphic for $p \neq q$.
 If $f \in N[\lambda I-P]$, for $\lambda \in \mathbb{T}$, it must satisfy $|Pf(x)|=|f(x)|$.
 $P$ is $L^{2}-L^{p:=2+\epsilon}$ hyperbounded, for some $\epsilon>0$. Then duality implies that $P^*$ is $L^{c(p)}-L^{2}$ hyperbounded, where $c(p)$ is the conjugate power of $p$. Now recall, $P \in\mathcal{B}(L^{p})$ and $dim[L^{p}\setminus Im(\lambda I-P)_{p}]= dim[N(\overline{\lambda}I-P^*)_{c(p)}]$. If $dim[N(\overline{\lambda}I-P^*)_{c(p)}]= \infty$, strict inclusion implies there exists a $\hat{g} \in L^{c(p)}\setminus L^{2}$ such that $P^* \hat{g}= \overline{\lambda} \hat{g}$ and $|P^* \hat{g}|=|\hat{g}|$ \emph{which contradicts hyperboundedness}. Therefore, $dim[L^p\setminus Im(\lambda I-P)_{p}]$ is finite which implies $Im(\lambda I-P)$ is closed in $L^{p}$and as $L^2 \subsetneq L^{c(p)}$ we have $dim[L^{2}\setminus Im(\lambda I-P)_{2}  ]=dim[N(\overline{\lambda}I-P^*)_{2}]< dim[N(\overline{\lambda}I-P^*)_{c(p)}]< \infty $ and we have that $Im(\lambda I-P)$ is closed when viewed as a map in $L^{2}$. Argument for $dim[N(\lambda I-P)_{2}]$ being finite follows the same hyperboundedness argument. Consequences follow from the disjoint nature of essential and discrete spectrum (see Theorem 7.9-7.11 of \cite{reed1980functional}).
\end{proof}

\begin{theorem}
\label{thm:specpoin}
If the pair $(P, \mu_{\pi})$ is ergodic, $L^2 - L^p$ hyperbounded for $p=3,4$ with norm condition $\|P\|_{L^2-L^p}<2^{\frac{1}{2}-\frac{1}{p}} $, then
\begin{equation}
    \label{eq:opcontoproj} \|P^{n}-U\|_{L^{2}(\mu_{\pi})} \longrightarrow 0, \hspace{3pt} \text{in a uniform operator topology}.
\end{equation}
\end{theorem}
\begin{proof}
The norm condition $\|P\|_{L^2 \rightarrow L^4} <2^{\frac{1}{2}-\frac{1}{4}}$ or $\|P\|_{L^2 \rightarrow L^3} <2^{\frac{1}{2}-\frac{1}{3}}$ ensures aperiodicity (see Theorem 5.1 and 5.3 in~\cite{cohen20222} ); only $\{1\} \in \sigma(P) \cap \mathbb{T}$ and corresponds to the eigenspan of constant functions. Following the argument of Theorem~\ref{thm:opcontoproj}: $P^n=S^{n} \bigoplus I_{N(I-P)} $, where $S=P_{|Im(I-P)}$ and  $\sigma(S) \subset \alpha \mathbb{D} $ for some $\alpha \in (0,1)$. Consequently, $\rho(S) \leq \alpha$, and for all $\rho \in (\alpha,1)$ there exists an $M(\rho)<\infty$ such that $\|S^n\| \leq M(\rho)\rho^{n}$; and thus, the result follows.
\end{proof}

\begin{remark}
A trivial corollary of  $\|P^{n}-U\|_{L^{2}(\mu_{\pi})} \longrightarrow 0$ {in uniform operator topology} is that for some $\rho \in (0,1)$, $C<\infty$ and for all $n \in \mathbb{N}$, it holds that
\begin{equation}
    \label{eq:uopcontopoinc}
    \|P^{n}\big(f-\mu_{\pi}(f)\big)\|_{L^2} \leq C \rho^{n} \|f-\mu_{\pi}(f)\|_{L^2}, 
\end{equation}
\end{remark}
which is equivalent to sharp decay of correlation in \eqref{eq:pointran}.

\section{Hyperboundedness and Transport-Entropy Inequality Implies Concentration}
\label{sec:Hyp}

After ensuring uncorrelation for time-distant samples, it is safe to relate to the iid setting and remind ourselves that \emph{sharp concentrations heavily rely on the ability to take exponential moments of the observable w.r.t underlying measure}. An attempt to analyse the uncorrelation 
and exponential integrability via a single linear operator brings us to the following discussion.

\paragraph{Feynman-Kac semigroup} plays an integral role in the study of fluctuations of time additive quantities of diffusion process in continuous time, \cite{touchette2018introduction}. The reader is referred to \cite{wang2020transport} for a detailed exposition on this topic in discrete-time setting. We identify with the observable/reward function $r$, a Feynman-Kac semigroup, which is a composition of the Markov transition operator followed by an exponentiated multiplication operator w.r.t observable -- i.e., $e^{r}P: D_{(\cdot)} \subseteq L^{2}(\mu_{\pi}) \rightarrow L^{2}(\mu_{\pi})$, 
such that for all $g \in D_{(\cdot)}$ (the domain on which semigroup operators can be viewed as bounded operator) it holds
\begin{flalign}
    \label{eq:fkacback}
    & (e^{r} P)^{n} g(x):= \mathbb{E}_{x} [g(x_N) e^{\sum_{i=0}^{n-1} r(x_i)} ] \hspace{5pt}, \forall n \in \mathbb{N}.
\end{flalign}
Assume that $x_0 \thicksim \beta$ and $\beta << \mu_{\pi}$. We have the following upper bound on the deviation of the observable: 
\begin{flalign}
     & \label{eq:dev} \mathbb{P}_{\beta} \bigg( \frac{1}{N} \sum_{i=0}^{N-1} r(x_i) -\mu_{\pi} (r) \geq \epsilon  \bigg)  \leq \bigg \| \frac{d \beta}{d \mu_{\pi}} \bigg \|_2  \inf_{s>0}\|(e^{sr}P)^{N}\|_{L^2 - L^2}e^{-sN( \mu_{\pi}(r)+\epsilon)}. 
\end{flalign}
Thus, the task at hand is bounding the operator $e^{r}P$  in a meaningful way to get a favorable concentration result. $P$ in itself is a positive contraction but exponentiated multiplication operator defined by unbounded observable require a short detour into its spectral analysis and transport-entropy inequality.

\paragraph{Multiplication operator defined by $e^r$.}
The operator $\big(e^{r},D(e^r)_{p}\big)$ is closed and densely defined, see e.g., Proposition 3.10 from Chapter 1 in \cite{engel2006short}. The essential range of $e^r$ with $\mu_{\pi}$ as an underlying measure is defined as $e^{r}_{ess} (\mu_{\pi}):= \bigg \{ \lambda \in [1, \infty) : \mu_{\pi} \big( |e^r -\lambda|< \epsilon \big) \neq 0, \hspace{2pt} \forall \hspace{2pt} \epsilon >0  \bigg\}$, and the essential norm of $e^r$, $\|e^r\|_{\infty}:= \sup \{ \lambda \in e^{r}_{ess} (\mu_{\pi}) \}
$. Multiplication operator $e^r$ is bounded in some and hence all $L^p$, iff  $\|e^r\|_{\infty} < \infty$. Consequently, $D(e^r)(p)=L^{p}$ and $\|e^r\|_{L^p \rightarrow L^p} = \|e^r\|_{\infty} $. If the stationary measure $\mu_{\pi}$ is not compactly supported and is absolutely continuous w.r.t Lebesgue measure then $\|e^r\|_{L^p \rightarrow L^p}= \infty$. However, not all is lost, as suggested in the previous section that we need some notion of hyperboundedness for uncorrelation. So if for some $p>2$, $P:L^{2} \rightarrow L^{p}$, in order to ensure $e^{r}P \in \mathcal{B}(L^{2})$ it is sufficient to prove $e^{r}:L^{p} \rightarrow L^{2}$. This is where the \emph{transport-entropy inequality} comes into play.

\begin{definition}
Consider metric space $(\mathcal{X},d)$ and reference  probability measure $\mu \in \mathcal{P}(\mathcal{X})$. Then we say that $\mu$ satisfies \emph{Transport-Entropy}(T-E) inequlaity with constant C or to be concise $\mu \in  \mathcal{T}_{1}^d (C)$ for some $C>0$ if for 
all $\nu \in \mathcal{P}(\mathcal{X})$ and $\nu << \mu$, it holds that 
\begin{equation}
    \label{eq:t1}
    \mathcal{W}_d (\mu, \nu) \leq \sqrt{2 C Ent(\nu||\mu)}.
\end{equation}
\end{definition}

\begin{lemma}[\cite{bobkov1999exponential}]
\label{lm:b-g}
$\mu$ satisfies $\mathcal{T}_{1}^d (C)$ if and only if for all Lipschitz function $f$ with $<f>_{\mu}:= \mathbb{E}_{\mu} f$, it holds that
\begin{flalign}
    \label{eq:bg} & \int e^{\lambda(f- <f>_{\mu})} d\mu \leq \exp(\frac{\lambda^2}{2}C \|f\|_{L(d)} ^2), \hspace{15pt} \text{where~~~~}  \|f\|_{L(d)}:= \sup_{x \neq y} \frac{|f(x)-f(y)|}{d(x,y)}.
\end{flalign}
\end{lemma}

\begin{theorem}
 \label{thm:ngtvehypebdd} If $\mu_{\pi} \in  \mathcal{T}_{1}^d (C_{\pi})$ for some $C_{\pi}>0$ and $r$ is Lipschitz w.r.t metric $d$ then we have for all  $p>2$, $e^{r}:L^{p} \longrightarrow L^{2}$ is a bounded operator with norm:
 \begin{equation}
 \label{eq:ngtvhyp}
 \|e^{r}\|_{L^p \rightarrow L^{2}} \leq \exp \bigg( \mu_{\pi}(r) + \frac{2p C_{\pi} \|r\|_{L(d)} ^2 }{(p-2)2} \bigg).    
 \end{equation}
 \end{theorem}
 
\begin{proof}
The proof is a simple application of Cauchy-Schwarz and the exponential moment inequality for Lipschitz functions under distribution $\mu_{\pi}$ satisfying the transport-entropy inequality. 
\end{proof}

\begin{theorem}
\label{thm:CRUXNOVEL}
Without any assumption of reversibility, if the invariant measure $\mu_{\pi} \in  \mathcal{T}_{1}^d (C_{\pi})$ for some $C_{\pi}>0$ and $r$ is Lipschitz w.r.t metric $d$ and for some $q>2$, $P:L^{2}(\mu_{\pi}) \longrightarrow L^{q}(\mu_{\pi})$ is hyperbounded with norm $\|P\|_{L^2 \rightarrow L^q } < e^{\frac{1}{2}\big[\frac{1}{2}-\frac{1}{q}\big]}$. Given $n \in \mathbb{Z}^{+}$ and $\delta \in (0,1)$ and an initial distribution $
\beta<< \mu_{\pi}$, if   $N \geq  \ln \bigg(\big\|\frac{d \beta}{\mu_{\pi}}\big\|_{2} \frac{1}{1-\delta}\bigg) \bigg( \frac{4n^{2}q}{(q-2) -4n^{2}q \ln \|P\|_{L^2 \rightarrow L^q } } \bigg) $, we have that
\begin{align}
    \label{eq:maincrux} \mathbb{P}_{\beta} \bigg( \frac{1}{N} \sum_{i=0}^{N-1} r(x_i) -\mu_{\pi} (r) \geq \frac{\sqrt{C_{\pi}}\|r\|_{L(d)}}{n}   \bigg)  < 1-\delta.
\end{align}
\end{theorem}

\begin{proof}
Since the theorem assumptions ensure $e^r P \in \mathcal{B}(L^{2}(\mu_{\pi})) $, we have that $\|(e^r P)^N\| \leq \|e^r P\|^{N}$. Let $\epsilon(n):= \frac{\sqrt{C_{\pi}}\|r\|_{L(d)}}{n} $,  where $\sqrt{C_{\pi}}\|r\|_{L(d)}$ is proportional to the square root of  variance of $r$ under distribution $\mu_{\pi}$. Notice that as opposed to standard $(\epsilon, \delta)$ probability arguments, here we have scaled $\epsilon=\epsilon(n)$ with a control variable $n$ that we can increase to make $\epsilon$ arbitrarily small. From the preceding discussion, it holds that 
\begin{align}
    \|(e^{sr}P)^{N}\|_{L^2 - L^2}e^{-sN( \mu_{\pi}(r)+\frac{\sqrt{C_{\pi}}\|r\|_{L(d)}}{n})} \leq \|P\|_{2-q} ^{N} e^{N\bigg( \big(\frac{2q}{q-2}\big) \frac{s^2}{2} C_{\pi} \|r\|_{L(d)}^2   - s\sqrt{C_{\pi}} \frac{\|r\|_{L(d)}}{n}\bigg)},
\end{align}
and a trivial calculation reveals
\begin{align}
    \inf_{s>0} \|P\|_{2-q} ^{N} e^{N\bigg( \big(\frac{2q}{q-2}\big) \frac{s^2}{2} C_{\pi} \|r\|_{L(d)}^2   - s\sqrt{C_{\pi}} \frac{\|r\|_{L(d)}}{n}\bigg)}= e^{N \bigg( ln \|P\|_{2-q} -\frac{1}{2n^2} \big[ \frac{q-2}{2q}\big]  \bigg)},
\end{align}
which we should be able to decrease as $N$ increase; for every $n \in \mathbb{Z}^{+}$, which happens to be the case if $\|P\|_{L^2 \rightarrow L^q } < e^{\frac{1}{2}\big[\frac{1}{2}-\frac{1}{q}\big]}$ and all other results follow.
\end{proof}
\begin{remark}
    Since our norm control $\|P\|_{L^2 \rightarrow L^q } < e^{\frac{1}{2}\big[\frac{1}{2}-\frac{1}{q}\big]}$ is in harmony with the upper bound for $q=3,4$ given by \cite{cohen20222}, which implies implies aperiodicity and consequently Poincare' $L^2-$ Spectral gap for Markov transition operator. We conjecture that for $q
    \in (2,3)$ our upper bound might imply aperiodicity and hence $L^2-$ Spectral gap.   
\end{remark}
Assuming hypercontractivity, a stronger   theoretical result similar to  \cite{wang2020transport} can be directly~deduced. 

\begin{corollary}
\label{cor: corhyptran}
Without any reversibility assumption on the pair $(P,\mu_{\pi})$, if the  stationary distribution satisfies T-E inequality i.e., $\mu_{\pi} \in \mathcal{T}_{1} ^{d} (C_{\pi})$ and for some $p>2$ associated transition kernel is hypercontractive, i.e., $\|P\|_{2 \rightarrow p} \leq 1 $, then for any initial distribution $\beta << \mu_{\pi}$ and $N \in \mathbb{N}$ it holds
\begin{flalign}
     & \label{eq:amcon} \mathbb{P}_{\beta} \bigg( \frac{1}{N} \sum_{i=0}^{N-1} r(x_i) -\mu_{\pi} (r) \geq \epsilon  \bigg)  \leq \bigg \| \frac{d \beta}{d \mu_{\pi}} \bigg \|_2  \exp{\bigg(-\frac{N \epsilon^2 (p-2)}{4 C_{\pi} \|r\|_{L(d)} ^2 p}  \bigg)}. 
\end{flalign}
Consequently, given $\epsilon>0$ and $\delta \in (0,1)$, it holds that for all $N \geq \frac{\ln \bigg(\big\|\frac{d \beta}{\mu_{\pi}}\big\|_{2} \frac{1}{1-\delta}\bigg)4 C_{\pi} \|r\|_{L(d)} ^2 p }{\epsilon^2 (p-2)}$  the chain satisfies $\mathbb{P}_{\beta} \bigg( \frac{1}{N} \sum_{i=0}^{N-1} r(x_i) -\mu_{\pi} (r) \geq \epsilon \bigg) \leq 1-\delta$.
\end{corollary}

\begin{example}
Let us verify these result on one dimensional linear Gaussian dynamical system with $|\alpha|<1$:
\begin{equation}
\label{eq:LG1}
    x_{n+1}= \alpha x_n +w_n. \hspace{5pt}, w_n \thicksim \mathcal{N}(0,1) \hspace{2pt} \text{and iid}.
\end{equation}
An easy check reveals stationary distribution of \eqref{eq:LG1} is $\gamma_{1,\alpha} := \mathcal{N}\bigg(0, \frac{1}{1- \alpha ^2}\bigg)$.

\begin{theorem}
\label{thm:gausshyp} Given $\alpha$ such that $|\alpha| <1 $, there exitst $p:=p(\alpha)>2$ such that $\|P\|_{2 \rightarrow p, \gamma_{1,\alpha}} \leq 1$  (Hypercontractivity) and $\|\big(e^{r}P\big)^{N}\| \leq e^{N \bigg( \gamma_{1,\alpha} (r) +\big(\frac{2p}{p-2}\big) \frac{\|r\|_{L(d)} ^2}{2} \bigg)} $.  
\end{theorem}

\begin{proof}
A trivial application of change of variable and Stein's lemma results in
\begin{flalign}
\label{eq:gausshyp}
\|P\|_{2 \rightarrow p} \leq \frac{1}{(1-\alpha^4)^{\frac{1}{4}}} \frac{1}{\bigg( 1-\frac{\alpha^2 p}{(1+\alpha^2)} \bigg)^{\frac{1}{2p}}} \frac{1}{e^{\frac{2}{p} \bigg( 1-\frac{\alpha^2 p}{(1+\alpha^2)} \bigg) }}.
\end{flalign}
It is a common knowledge (see Corollary 7.2 of \cite{gozlan2010transport}), that$\gamma_{1, \alpha}  \in \mathcal{T}_{1} ^{d} (1)$ i.e., satisfies T-E 
inequality with constant 1. It follows from~\eqref{eq:gausshyp} that for all $g \in L^{2} \big(\gamma_{1,\alpha}\big)$ such that $\|g\|_{2} \leq 1 $, there exists  $p>2$ such that $\|Pg\|_{p} \leq 1$. By applying Cauchy-Schwarz to the powers $\frac{p}{2}$ and its' conjugate number $\frac{p}{p-2}$ we get:
\begin{flalign}
 & \label{eq:hyp+TE} \|\big(e^{r}P\big)g\|_{2}  \leq \|Pg\|_{p}  \bigg( \int e^{\frac{2p}{p-2} r(x)} d \gamma_{1,\alpha} (x) \bigg)^{\frac{p-2}{2p}} 
  \leq \|P\|_{2 \rightarrow p }  \bigg( e^{ \gamma_{1,\alpha} (r)} e^{\big(\frac{2p}{p-2}\big) \frac{\|r\|_{L(d)} ^2 }{2}} \bigg),
\end{flalign}
where \eqref{eq:hyp+TE} follows from hypercontractivity of the transition operator and the fact that stationary distribution satisfies T-E inequality with $C=1$ and $d$ is the Euclidean metric. Conclusion follows from the trivial inequality $\|\big(e^{r}P\big)^{N}\| \leq \|\big(e^{r} P \big)\|^{N}$ and we have a sharp concentration as in Corollary~\ref{cor: corhyptran}.
\end{proof}

\end{example}
\section{Conclusion and Future Work }
\label{sec:conclusion}
In this work, we have narrowed down the concentration phenomenon for Harris ergodic Markov chains to a study of the composition of the Markov transition operator followed by a an exponentiated multiplication operator defined by the observable under consideration. Hyperboundedness and transport-entropy inequality suffices for  concentration phenomen. However, there are still unanswered questions that needs further exploration. 
For example, the impact of \emph{aperiodicity and hyperboundedness} -- we believe that via continuity of the \emph{Fredholm index} we can extend the norm control on $\|P\|_{L^2- L^{p}}$ that implies \emph{aperiodicity}. Currently, only result for $p=3,4$ is available. Although the transport-entropy inequality can be verified via exponential-type Lyapunov function we introduced 
in the accompanying paper~\cite{naeem2022transportation} easily verifiable conditions for hyperboundedness on the continuous state space is still an open problem (besides the linear Gaussian case). Extending analysis to estimation of steady state correlation is also a work in progress.

\section{Appendix}
\label{sec:appendix}

\paragraph{Spectral Decomposition}

\paragraph{Preliminaries.}
Let $A$ be a bounded operator on a Banach space $\mathcal{X}$, denoted by $A \in \mathcal{B}(\mathcal{X})$. $\lambda \in \mathbb{C}$ is said to be in resolvent of $A$, denoted by $\lambda \in RS(A)$, if $\lambda I -A : \mathcal{X} \rightarrow \mathcal{X}$ is bijective. Bounded inverse theorem implies that for $\lambda \in RS(A)$, $R_{\lambda}(A):=(\lambda I-A)^{-1} \in \mathcal{B}(\mathcal{X})$ . The spectrum of $A$ is denoted by $\sigma(A):= \mathbb{C} \setminus RS(A)$. Spectral radius of $A$ is denoted by $\rho(A):= \sup{|\lambda|: \lambda \in \mathbb{C}}$.  Consequently, if there exists a $v \in \mathcal{X}$ and $\lambda \in \mathbb{C}$ such that $Av= \lambda v$ (i.e., $\lambda$ is an eigenvalue with the corresponding eigenfunction $v$) then $\lambda I -A$ is not injective, which implies $\lambda \in \sigma(A)$. However, the spectrum of $A$ is not limited to eigenvalues; for a detailed monograph of spectral theory see e.g., \cite{reed1980functional}.

\begin{definition}
Let $A$ be a bounded operator on some Banach space $\mathcal{X}$. It is called \emph{Fredholm} if: (a) $ dim[N(A)]< \infty$, (b) $ dim[\mathcal{X}\setminus Im(A)]< \infty$, and (c) $ Im(A)$ is closed. Here, $dim[N(A)]$ denotes the dimension of the null space of $A$ and $Im(A)$ means the range of $A$; $ dim[\mathcal{X}\setminus Im(A)]$ is often read as the dimension of \emph{cokernel} of $A$.
\end{definition}
Note that the condition (c) is redundant; it follows from (b). 

\begin{definition}
The index of a Fredholm operator $A$ is defined as
$ind(A)=dim[N(A)]-dim[X\setminus Im(A)]$, and for some small perturbations $\Delta(A)$ to the operator $ind(A+\Delta(A))=ind(A)$.
\end{definition}
\begin{definition}
$\lambda \in \sigma(A)$ is said to be in essential spectrum of $A$, \emph{($\lambda \in \sigma_{ess}(A)$ )}  if and only if $\lambda I- A$ is not a Fredholm operator. $\lambda \in \sigma(A)\setminus \sigma_{ess}(A) $ is said to be in discrete spectrum, \emph{($\lambda \in \sigma_{disc}(A)$ )}  if and only if (a) $\lambda$ is an isolated point of $\sigma(A)$ and (b) $\{\psi \in \mathcal{X}:A\psi=\lambda \psi \}$ is finite dimensional.
\end{definition}

\begin{definition}
\label{def:hypbd}
    Markov operator is called \emph{hyperbounded}, if for some $1\leq p<q \leq \infty$, $P$ is a bounded operator from $L^{p}(\mu_{\pi})$ to $L^{q}(\mu_{\pi})$ (we will often call it $L^{p}-L^{q}$ hyperboundedness). More stringent requirement of \emph{hypercontractivity} requires $P$ to be hyperbounded with $\|P\|_{L^p \rightarrow L^q} \leq 1$. 
\end{definition}

\begin{remark}
    It follows from Jensen's inequality that for all $r \geq 1$, $\|P\|_{L^{r}(\mu_{\pi})} \leq 1$, and Riesz-Thorin interpolation implies that for all $1<p<q$, $P$ is $L^{p}-L^{q}$ hyperbounded. 
\end{remark}
\bibliography{l4dc2023.bib}

\end{document}